\documentclass{article}

\usepackage{times}
\usepackage{graphicx} 
\usepackage{subfigure} 

 \usepackage{natbib}

\usepackage{algorithm}
\usepackage{algorithmic}

\usepackage{hyperref}

\usepackage[accepted]{icml2019}

\usepackage{graphicx,mdframed}
\usepackage{hyperref,url}
\usepackage{tabularx}
\usepackage{ctable}
\usepackage{multirow}
\usepackage{booktabs}
\usepackage{rotating}
\usepackage{mdframed}
\usepackage[symbols,nogroupskip]{glossaries-extra}

\newsavebox\mybox
\newlength\mylength

\usepackage{amsmath,amssymb,amsthm}

\newtheorem{theorem}{Theorem}
\newtheorem{lemma}{Lemma}

\theoremstyle{definition}

\DeclareMathOperator{\acc}{acc}
\DeclareMathOperator{\err}{err}

\DeclareMathOperator{\sign}{sign}

\DeclareMathOperator{\kl}{kl}

\DeclareMathOperator{\Bern}{Bern}

\DeclareMathOperator{\osc}{osc}
\DeclareMathOperator{\ball}{Ball}

\DeclareMathOperator{\supp}{supp}
\DeclareMathOperator{\ent}{Ent}
\DeclareMathOperator{\mi}{MI}

\DeclareMathOperator{\xadv}{x^{\text{adv}}}

\newtheorem{corollary}{Corollary}
\newtheorem{definition}{Definition}
\newtheorem{remark}{Remark}

\begin{document}

\twocolumn[
\icmltitle{Limitations of Adversarial Robustness: Strong No Free Lunch Theorem}



\icmlsetsymbol{equal}{*}

\begin{icmlauthorlist}
  \icmlauthor{Elvis Dohmatob}{aff}
\end{icmlauthorlist}

\icmlaffiliation{aff}{Criteo, Paris, France.}

\icmlcorrespondingauthor{Elvis Dohmatob}{e.dohmatob@criteo.com}

\icmlkeywords{boring formatting information, machine learning, ICML}

\vskip 0.3in
]



\printAffiliationsAndNotice 

\begin{abstract}
This manuscript presents some new impossibility results on adversarial robustness
in machine learning, a very important yet largely open problem. We show that if
conditioned on a class label the data distribution satisfies the $W_2$ Talagrand
transportation-cost inequality (for example, this condition is satisfied if the
conditional distribution has density which is log-concave; is the uniform
measure on a compact Riemannian manifold with positive Ricci curvature; etc.)
any classifier can be adversarially fooled with high probability once the
perturbations are slightly greater than the natural noise level in the problem.
We call this result The Strong "No Free Lunch" Theorem as some recent results
(Tsipras et al. 2018, Fawzi et al. 2018, etc.) on the subject can be immediately
recovered as very particular cases. Our theoretical bounds are demonstrated on
both simulated and real data (MNIST). We conclude the manuscript with some
speculation on possible future research directions.
\end{abstract}

\section{Introduction}
\label{sec:intro}
An adversarial attack operates as follows:
\begin{itemize}
\item A classifier is trained and deployed (e.g the road traffic sign
  recognition system on a self-driving car).  
\item At test / inference time, an attacker may submit queries to the
  classifier by sampling a data point $x$ with true label $k$, and
  modifying it $x \rightarrow x^{\text{adv}}$ according to a prescribed threat
  model. For example, modifying a few pixels on a road traffic sign
  ~\cite{onepixel}, modifying intensity of pixels by a limited amount determined
  by a prescribed tolerance level $\epsilon$ ~\cite{tsipras18}, etc. 
  $\epsilon$, on it.
\item The goal of the attacker is to fool the classifier into classifying
  $x^{\text{adv}}$ as label different from $k$.
\item A robust classifier tries to limit this failure mode, at a prescribed
  tolerance $\epsilon$.
\end{itemize}

\subsection{A toy example illustrating the fundamental issue}
\label{sec:toy}
To motivate things, consider the following "toy" problem from ~\cite{tsipras18},
which consists of classifying a target $Y \sim {\Bern}(1/2,\{\pm 1\})$ based on
$p\ge 2$ explanatory variables $X:=(X^1,X^2,\ldots,X^p)$ given by
$$X^1 | Y = \begin{cases}+Y,&\mbox{w.p } 70\%,\\
-Y,&\mbox{ w.p } 30\%,\end{cases}
$$
and
$X^j | Y \sim \mathcal N (\eta Y,1), \text{for } j=2,\ldots,p$, where $\eta \sim
p^{-1/2}$ is a fixed scalar which (as we wll see) controls the difficulty of the
problem. Now, as was shown in ~\cite{tsipras18}, the above problem can be solved
perfectly with generalization accuracy $\approx 100\%$, but the "champion"
estimator can also be fooled, perfectly! Indeed, the linear estimator given by
$h_{\text{avg}}(x) := \sign({w}^Tx)$ with ${w} = (0, 1/(p-1), \ldots, 1/(p-1))
\in \mathbb R^p$, where we allow an attacked to modify each feature by an amount
at moust $\epsilon \approx 2\eta$, has the afore-mentioned properties. Indeed,
using basic tail bounds for the Gaussian distributions, one can show that for
any $\delta \in (0, 1]$ the following hold
\begin{itemize}
  \item The standard accuracy of the linear model $h_{\text{avg}}$ is at least
    $1 - \delta$ if $\eta \ge \sqrt{2\log(1/\delta)/(p-1)}$, and
  \item This same model's adversarial accuracy is at most $\delta$ for $\epsilon
    \ge \eta + \sqrt{2\log(1/\delta)/(p-1)}$
  \end{itemize}
 See ~\cite{tsipras18} for details (or see supplemental).

By the way, we note that an optimal adversarial attack can be done by taking
$\Delta x^1=0$ and $\Delta x^j = -\epsilon y$ for all $j=2,\ldots,p$.
\paragraph*{An autopsy of what is going on.}
Recall that the entropy of a univariate Gaussian is $\ent(\mathcal
N(\mu,\sigma))=\ln(\sqrt{2\pi\sigma e})$ nats. Now, for $j=2,3,\ldots,p$, the
distribution of feature $X^j$ is a Gaussian mixture $\frac{1}{2}\sum_{Y=\pm
  1}\mathcal N(\eta Y,1)$ and so one computes the mutual information between
$X^j$ and the class label $Y$ as
\begin{eqnarray*}
\begin{split}
  &\mi(X^j;Y) := \ent(X^j)-\ent(X^j|Y)\\
  &=\ent \left(\frac{1}{2}\sum_{y=\pm 1}\mathcal N(\eta
    y,1)\right) - \frac{1}{2}\sum_{y=\pm1}\ent(\mathcal N(\eta y, 1))\\
  &=\ln(\sqrt{2\pi e})+\eta^2 - r - 2(1/2)\ln(\sqrt{2\pi e})
  = \eta^2-r \le \eta^2,
\end{split}
\end{eqnarray*}
where (see ~\cite{Michalowicz08} for the details)
\begin{eqnarray*}
  r:=\frac{2}{\sqrt{2\pi}\eta}e^{-\eta^2/2}\int _0^\infty
  e^{-\frac{z^2}{2\eta^2}}\cosh(z)\ln(\cosh(z))dz \ge 0.
\end{eqnarray*}
Thus $\mi(X^j;Y) \le \eta^2$. 
Since $\eta^2 \sim 1/p$, we conclude that these features barely share any
information with the target variable $Y$. Indeed, ~\cite{tsipras18} showed
improved robustness on the above problem, with feature-selection based on
mutual information.




\paragraph{Basic ``No Free Lunch'' Theorem.}
Reading the information calculations above, a skeptic could point out that the
underlying issue here is that the estimator $h_{\text{avg}}$ over-exploits the
fragile / non-robust variables $X^2,\ldots,X^p$ to boost ordinary generalization
accuracy, at the expense of adversarial robustness. However, it was rigorously
shown in ~\cite{tsipras18} that on this particular problem, every estimator is
vulnerable. Precisely, the authors proved the following basic ``No Free Lunch''
theorem.

\begin{theorem}[Basic No Free Lunch, ~\cite{tsipras18}] For the problem above,
  \textbf{any} estimator which has ordinary accuracy \textbf{at least} $1 -
  \delta$ must have robust adversarial robustness accuracy \textbf{at most}
  $7\delta/3$ against $\ell_\infty$-perturbations of maximum size $\epsilon \ge
  2\eta$.
\end{theorem}


\subsection{Highlight of our main contributions}
In this manuscript, we prove that under some ``curvature conditions'' (to be
precised later) on the conditional density of the data, it holds that
  
For geodesic / faithful attacks:
  \begin{itemize}
  \item Every (non-perfect) classifier can be adversarially fooled with high
    probability by moving sample points an amount less than a critical value,
    namely
    $$
    \epsilon(h|k) :=\sigma_k\sqrt{2\log(1/\err(h|k))} \approx \sigma_k\Phi^{-1}(\acc(h|k))
    $$
    along the data manifold, where $\sigma_k$ is the ``natural noise level'' in
    the data points with class label $k$ and $\err(h|k)$ generalization error of
    the classifier in the non-adversarial setting
    
  \item Moreover, the average distance of a
    sample point of true label $k$ to the error set is upper-bounded by
    $$
    \epsilon(h|k) + \sigma_k\sqrt{\frac{\pi}{2}} =
    \sigma_k\left(\Phi^{-1}(\acc(h|k))+\sqrt{\frac{\pi}{2}}\right)
    $$
  \end{itemize}
For attacks in flat space $\mathbb R^p$:
\begin{itemize}  
  \item In particular, if the data points live in $\mathbb R^p$, where $p$ is
  the number of features), then every classifier can be adversarially fooled
  with high probability, by changing each feature by an amount less than a
  critical value, namely
  \[
    \begin{split}
      \epsilon_\infty(h|k) &:= \sigma_k\sqrt{2\log(1 /
        \err(h|k))/p}\\
      &\approx \frac{\sigma_k}{\sqrt{p}}\Phi^{-1}(\acc(h|k)).  
    \end{split}
  \]
\item Moreover, we have the bound
  \[
    \begin{split}    
    d(h|k) &\le \epsilon_\infty(h|k) + \frac{\sigma_k}{\sqrt{p}}
    \sqrt{\frac{\pi}{2}}\\
    &\approx \frac{\sigma_k}{\sqrt{p}}\left(\Phi^{-1}(\acc(h|k)) +
      \sqrt{\frac{\pi}{2}}\right).
    \end{split}
  \]
\end{itemize}

In fact, we prove similar results for $\ell_1$ (reminiscent of few-pixel attacks
~\cite{onepixel}) and even any $\ell_s$ norm on $\mathbb R^p$. We call these
results The Strong ``No Free Lunch'' Theorem as some recent results  (e.g
~\cite{tsipras18,fawzi18,gilmerspheres18}), etc.)  on the subject can be
immediately  recovered as very particular cases. Thus adversarial
(non-)robustness should really be thought of as a measure of complexity of a
problem. A similar remark has been recently made in ~\cite{bubeck2018}.

The sufficient ``curvature conditions'' alluded to above imply
\textit{concentration of measure phenomena}, which in turn imply our
impossibility bounds. These conditions are satisfied in a large number of
situations, including cases where the class-conditional distribution is the
volume element of a compact Riemannian manifold with positive Ricci curvature;
the class-conditional data distribution is supported on a smooth manifold and has
log-concave density w.r.t the curvature of the manifold; or the manifold is
compact; is the \textit{pushforward} via a Lipschitz continuous map, of another
distribution which verifies these curvature conditions; etc.
\subsection{Notation and terminology}
$\mathcal X$ will denote the feature space and $\mathcal Y:=\{1,2,\ldots,K\}$
will be the set of class labels, where $K \ge 2$ is the number of classes, with
$K=2$ for binary classification. $P$ will be the (unknown) joint probability
distribution over $\mathcal X \times \mathcal Y$, of two prototypical random
variables $X$ and $Y$ referred to the features and the target variable, which
take values in $\mathcal X$ and $\mathcal Y$ respectively. Random variables will
be denoted by capital letters $X$, $Y$, $Z$, etc., and realizations thereof will
be denoted $x$, $y$, $z$, etc. respectively.

For a given class label $k \in \mathcal Y$, $\mathcal X_k \subseteq \mathcal X$ will
denote the set of all samples whose label is $k$ with positive probability under
$P$. It is the support of the restriction of $P$ onto the plane $\mathcal X
\times \{k\}$. This restriction is denoted $P_{X|Y=k}$ or just $P_{X|k}$, and
defines the \text{conditional} distribution of the features $X$ given that the
class label has the value $k$. We will assume that all the $\mathcal X_k$'s are
finite-dimensional smooth Riemannian manifolds. This is the so-called
\textit{manifold assumption}, and is not unpopular in machine learning
literature. A classifier is just a \textit{measurable} mapping $h: \mathcal X
\rightarrow \mathcal Y$, from features to class labels.


\paragraph*{Threat models.}
Let $d_{\mathcal X}$ be a distance / metric on the input space  $\mathcal
X$ and $\epsilon \ge 0$ be a tolerance level. The  $d_{\mathcal X}$
\textit{threat model} at tolerance $\epsilon$ is a scenario where the attacker
is allowed to perturb any input point $x \mapsto \xadv$, with the constraint
that $d_{\mathcal X}(\xadv,x) \le \epsilon$. When $\mathcal X$ is a manifold,
the threat model considered will be that induced by the \textit{geodesic
  distance}, and will be naturally referred to as the \textit{geodesic threat model}.
\paragraph*{Flat threat models.}
  In the special case of euclidean space $\mathcal X = \mathbb R^n$,
  we will always consider the distances defined for $q \in [1, \infty]$ by
$d(x,z) = \|x-z\|_q$, where
\begin{eqnarray}
  \|a\|_q :=
  \begin{cases}
    \left(\sum_{j=1}^p|a^j|^q \right)^{1/q},&\mbox{ if }1 \le q < \infty,\\
    \max\{|a^1|,\ldots,|a^p|\},&\mbox{ if }q = \infty.
  \end{cases}
\end{eqnarray}
The $\ell_\infty$ / \textit{sup} case where $q=\infty$~\cite{tsipras18} is
particularly important: the corresponding threat model allows the adversary
to separately increase or decrease each feature by an amount at most $\epsilon$.
The \textit{sparse} case $q=1$ is a convex proxy for so-called ``few-pixel''
attacks~\cite{onepixel} wherein the total number of features that can be
tampered-with by the adversary is limited.

\paragraph*{Adversarial robustness accuracy and error.} The
\textit{adversarial robustness accuracy} of $h$ at tolerance $\epsilon$ for a
  class label $k \in \mathcal Y$ and w.r.t the $d_{\mathcal X}$ threat model,
  denoted $\acc_{d_{\mathcal X},\epsilon}(h|k)$, is defined by
\begin{eqnarray}
  \acc_{d_{\mathcal X},\epsilon}(h|k) :=
    P_{X|k}(h(x')=k \;\forall x'\in \ball_{\mathcal X}(X; \epsilon)).
\end{eqnarray}
This is simply the probability no sample point $x$ with true class label $k$
can be perturbed by an amount $\le \epsilon$ measured by the distance
$d_{\mathcal X}$, so that it get misclassified by $h$. This is an adversarial
version of the standard class-conditional accuracy $\acc(h) = P_{(X,Y)}(h(X)=Y)$
corresponding to $\epsilon = 0$. The corresponding \textit{adversarial
  robustness error} is then $\err_\epsilon(h|k) := 1-\acc_\epsilon(h|k)$.
This is the adversarial analogue of the standard notion of the class-conditional
\textit{generalization / test error}, corresponding to $\epsilon = 0$.

Similarly, one defines the \textit{unconditional adversarial accuracy}
  \begin{eqnarray}
    \acc_\epsilon(h) = P_{(X,Y)}(h(x')=Y\;\forall x' \in
    \operatorname{Ball}_{\mathcal X}(X;\epsilon)),
  \end{eqnarray}
  which is an adversarial version of the standard accuracy
  $\acc(h) = P_{(X,Y)}(h(X)=Y)$.
  Finally, \textit{adversarial robustness radius} of $h$ on class $k$
  \begin{eqnarray}
    d(h|k) := \mathbb E_{X|k}[d(X, B(h,k))],
    \end{eqnarray}
where $B(h,k):=\{x \in \mathcal X | h(x) \ne k\}$ is the set of inputs
classified by $h$ as being of label other than $k$. Measureablity of $h$ implies
that $B(h,k)$ is a Borel subset of $\mathcal X$. $d(x, k)$ is nothing but the
average distance of a sample point $x \in \mathcal X$ with true label $k$, from
the set of samples classified by $h$ as being of another label. The smaller the
value of $d(h|k)$, the less robust the classifier $h$ is to adversarial attacks
on samples of class $k$.
  
\begin{remark}
By the properties of expectation and conditioning, it holds that
$\min_{k}\acc_{\epsilon}(h|k) \le \acc_{\epsilon}(h) = \mathbb
E_{Y}[\acc_{\epsilon}(h|Y)] = \sum_{k=1}^K\pi_{k}\acc_{\epsilon}(h|k)
\le \max_k\acc_{\epsilon}(h|k)$, where $\pi_k := P(k)$. Thus, bounds on the
$\acc_\epsilon(h|k)$'s imply bounds on $\acc_\epsilon(h)$.
\end{remark}

  \subsection{Rough organization of the manuscript}
  In section \ref{sec:toy}, we start off by presenting a simple motivating
  classification problem from ~\cite{tsipras18}, which as shown by the
  authors, already exhibits the ``No Free Lunch'' issue. In section
  \ref{sec:term} we present some relevant notions from geometric probability
  theory which will be relevant for our work, especially Talagrand's
  \textit{transportation-cost} inequality and also Marton's blowup Lemma.
Then in section \ref{sec:nfl}, we present the main result of this
manuscript, namely, that on a rich set of distributions no classifier
can be robust even to modest perturbations (comparable to the natural noise
level in the problem). This generalizes the results of
~\cite{tsipras18,gilmerspheres18} and to some extent, ~\cite{fawzi18}.
Our results also extend to the distributional robustness setting (no highlighted
here but is presented in the Appendix \ref{sec:dr}).

An in-depth presentation of related works is given in section
\ref{sec:related-work}.
Section \ref{sec:experiments} presents experiments on both simulated and real
data that confirm our theoretical results. Finally, section \ref{sec:conclusion}
concludes the manuscript with possible future research directions.

All proofs are presented in Appendix \ref{sec:proofs}.

\section{The theorems}
\subsection{Terminology and background}
\label{sec:term}
\paragraph{Neighborhood of a set in a metric space.}
  The $\epsilon$-blowup (aka $\epsilon$-neighborhood, aka $\epsilon$-fattening,
  aka $\epsilon$-enlargement) of a subset $B$ of a metric space $\mathcal
  X=(\mathcal X, d_{\mathcal X})$, denoted $B_{\mathcal X}^\epsilon$, is defined
  by $B_{\mathcal X}^\epsilon := \{x \in \mathcal X | d_{\mathcal X}(x, B) \le
  \epsilon\}$, where $d_{\mathcal X}(x,B) := \inf\{d_{\mathcal X}(x,y)|y \in
  B\}$ is the distance of $x$ from $B$.
   Note that $B_{\mathcal X}^\epsilon$ is an increasing function of both $B$ and
   $\epsilon$; that is, if $A \subseteq B \subseteq \mathcal X$ and $0 \le
   \epsilon_1 \le \epsilon_2$, then $A \subseteq A_{\mathcal X}^{\epsilon_1}
   \subseteq B_{\mathcal X}^{\epsilon_1} \subseteq B_{\mathcal X}^{\epsilon_2}$.
   In particular, $B_{\mathcal X}^0 = B$ and $B_{\mathcal X}^\infty = \mathcal X$.
   Also observe that each $B_{\mathcal X}^\epsilon$ can be rewritten in the form
   $B_{\mathcal X}^\epsilon
       = \bigsqcup_{x \in B}\operatorname{Ball}_{\mathcal X}(x;\epsilon)$,
   where
   $ \operatorname{Ball}_{\mathcal X}(x;\epsilon) := \{x' \in \mathcal X |
   d_{\mathcal X}(x',x) \le \epsilon\}$ the closed ball in $\mathcal X$ with
   center $x$ and radius $\epsilon$. Refer to Fig. \ref{fig:blowup}.
\begin{figure}[!h]
  \includegraphics[width=.5\textwidth]{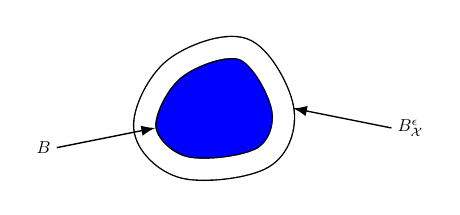}
  \caption{$\epsilon$-blowup of a subset $B$ of a metric space $\mathcal X$.
    \label{fig:blowup}
  }
\end{figure}
In a bid to simplify notation, when there is no confusion about the underlying
metric space, we will simply write $B^{\epsilon}$ for $B_{\mathcal X}^\epsilon$.
When there is no confusion about the the underlying set $\mathcal X$ but not the
metric thereupon, we will write $B_{d_{\mathcal X}}^\epsilon$. For example, in the
metric space $(\mathbb R^p,\ell_q)$, we will write $B_{\ell_q}^\epsilon$ instead
of $B_{(\mathbb R^p,\ell_q)}^\epsilon$ for the $\epsilon$-blowup of $B \subseteq
\mathbb R^p$.

An example which will be central to us is when $h: \mathcal X \rightarrow
\mathcal Y$ is a classifier, $k \in \mathcal Y$ is a class label, and we take
$B$ to be the ``bad set'' $B(h,k) \subseteq \mathcal X$ of inputs which are
assigned a label different from $k$, i.e
\begin{eqnarray}
  \begin{split}
    B(h,k):=\{x | h(x) \ne k\}
    = \bigsqcup_{k' \ne
    k}\{x| h(x) = k'\}.
  \end{split}
\end{eqnarray}
$B_{\mathcal X}^\epsilon=B(h,k)^\epsilon$ is then nothing but the event that there is data point
with a ``bad $\epsilon$-neighbor'', i.e the example can be missclassified by
applying a small perturbation of size $\le \epsilon$.
 This interpretation of blowups will be central in the sequel, and we will be
 concerned with lower-bounding the probability of the event $B(h,k)^\epsilon$
 under the conditional measure $P_{X|k}$. This is the proportion of points $x
 \in \mathcal X$ with true class label $k$, such that $h$ assigns a label $\ne
 k$ to some  $\epsilon$-neighbor $x' \in \mathcal X$ of $x$. Alternatively,  one
 could study the local \textit{robustness radii} $r_h(x,k) := \inf\{d(x',x) | x'
 \in \mathcal X,\;h(x') \ne k\} =: d(x,B(h,k))$, for $(x,k) \sim P$,  as was
 done in ~\cite{fawzi18}, albeit for a very specific problem setting
 (generative models with Guassian noise). More on this in section
 \ref{sec:related-work}. Indeed $r_h(x,k) \le \epsilon \iff x \in B(h,k)^\epsilon$.
 

 \subsection{Isoperimetric inequalities in general metric spaces}
 All known theoretical works ~\cite{tsipras18, schmidt2018, gilmerspheres18,
   fawzi18, saeed2018, goldstein} on the impossibility of adversarial robustness
 use the Gaussian isoperimetric inequality in euclidean space
 ~\cite{boucheron2013}. Via the classical work of L\'evy-Gromov, one can extend
 this to compact simply connected Riemannian manifolds with positive Ricci
 curvature. For general metric spaces, with which we are concerned, it is not
 even clear what should be the appropriate notion of curvature. Villani and
 co-workers have solved this issue by developing a so-called \emph{synthetic}
 theory of curvature, based on \emph{optimal transport}. The resulting geometric
 framework allows for an extension (due mostly to F. Otto and C. Villani, but
 also M. Talagrand and K. Marton) of L\'evy-Gromov theory to metric
 spaces~\cite{Villani}.
 
Fix a reference measure $\mu$ in $\mathcal P^2(\mathcal X)$, the
\textit{Wasserstein} space of all probability measures $\mu$ on the metric space
$\mathcal X$ with finite order moment, i.e such that there exists a point $a \in
\mathcal X$ with $\mathbb E_{X\sim \mu}[d_{\mathcal X}(a,X)^2] < \infty$. Let $c
\ge 0$.

\begin{definition}[$\text{T}_2(c$) property --a.k.a Talagrand $W_2$
  transportation-cost inequality] $\mu$ is said to satisfy $\text{T}_2(c)$ if
  for every other distribution $\nu$ on $\mathcal X$, which is absolutely
  continuous w.r.t $\mu$ (written $\nu \ll \mu$), one has
  \begin{eqnarray}
    W_2(\nu,\mu) \le \sqrt{2c\kl(\nu\|\mu)},
    \label{eq:t2}
  \end{eqnarray}
  where $W_2(\nu,\mu) := \left(\underset{\pi \in \Pi(\nu,\mu)}{\inf}\mathbb E_\pi
        [d_{\mathcal X}(X', X)^2]\right)^{1/2}$
is the Wasserstein $2$-distance between $\nu$ and $\mu$, with $\Pi(\nu,\mu)$
being the set of all couplings of $\nu$ and $\mu$; and
      $\kl(\nu\|\mu)$ is the entropy of $\nu$ relative to $\mu$, defined
        by $\kl(\nu\|\mu) = \int_{\mathcal X} \log(\frac{d\nu}{d\mu})d\mu$ if
        $\nu \ll \mu$ and $+\infty$ else.
\end{definition}
Note that if $0 \le c \le c'$, then $\text{T}_2(c) \subseteq \text{T}_2(c')$.
The inequality \eqref{eq:t2} in the above definition is a generalization of
the well-known \textit{Pinsker's inequality} for the total variation distance
between probability measures. Unlike Pinsker's inequality which holds
unconditionally, \eqref{eq:t2} is a privilege only enjoyed by special classes
of reference distributions $\mu$. These include: log-concave distributions
on manifolds (e.g multi-variate Gaussian), distributions on compact
Riemannian manifolds of positive Ricci curvature (e.g spheres, tori,
etc.), pushforwards of distributions that satisfy some $\text{T}_2$ inequality,
etc. In section \ref{sec:apps}, these classes of distributions will be discussed
in detail as sufficient conditions for our impossibility theorems.

\begin{definition}[BLOWUP($c$) property] $\mu$ is said to satisfy BLOWUP($c$) if
  for every Borel $B \subseteq \mathcal X$ with $\mu(B) > 0$ and for every
  $\epsilon \ge \sqrt{2c\log(1/\mu(B))}$, it holds that
  \begin{eqnarray}
    \mu(B^\epsilon) \ge 1 -
    e^{-\frac{1}{2c}(\epsilon-\sqrt{2c\log(1/\mu(B))})^2}.
    \label{eq:blowup}
  \end{eqnarray}
\end{definition}
It is a classical result that the Gaussian distribution on
$\mathbb R^p$ has BLOWUP($1$) and $\text{T}_2(1)$, a phenomenon known as
\textit{Gaussian isoperimetry}. 
These results date back to at least works of E. Borel, P. L\'evy, M. Talagrand
and of K. Marton ~\cite{boucheron2013}.
 
The following lemma is the most important tool we will use to
derive our bounds.

\begin{lemma}[Marton's Blowup lemma]
  On a fixed metric space, it holds that $\text{T}_2(c) \subseteq BLOWUP(c)$.
  \label{thm:blowup}
\end{lemma}
\begin{proof}
  The proof is classical, and is provided in Appendix \ref{sec:proofs} for the
  sake of completeness.
\end{proof}
The Otto-Villani theory ~\cite{Villani} showed that there are even more primitive so-called
\emph{log-Sobolev} inequalities which imply $\operatorname{T}_2$ inequalities,
and therefore by Lemma \ref{thm:blowup}, imply measure concentration.

\subsection{Generalized ``No Free Lunch'' Theorem}
It is now ripe to present the main results of this manuscript.
\label{sec:nfl}

\begin{theorem}[Strong ``No Free Lunch'' on curved space]
Suppose that for some $\sigma_k > 0$, 
  $P_{X|k}$ has the $\text{T}_2(\sigma_k^2)$ property on the conditional
  manifold $\mathcal X_k:= \supp(P_{X|k}) \subseteq \mathcal X$. Given a
  classifier $h: \mathcal X \rightarrow \{1,2,\ldots,K\}$ for which $\acc(h|k) <
  1$ (i.e the classifier is not perfect on the class $k$), define
  \begin{eqnarray}
    \epsilon(h|k) := \sigma_k\sqrt{2\log(1/\err(h|k))}\approx\sigma_k\Phi^{-1}(\acc(h|k)).
    \label{eq:critical}
    \end{eqnarray}
    Then for the geodesic threat model, we have the bounds
      
    \textbf{\textit{(A)} Adversarial robustness accuracy:} If $\epsilon \ge
    \epsilon(h|k)$, then
  \begin{eqnarray}
    \acc_{\epsilon}(h|k) \le
    \min(\acc(h|k),e^{-\frac{1}{2\sigma_k^2}(\epsilon -
    \epsilon(h|k))^2}).
    \label{eq:bingo}
    \end{eqnarray}
 
\textbf{\textit{(B)} Bound on average distance to error set:}
  \begin{eqnarray}
    \begin{split}
      d(h|k) &\le \sigma_k\left(\sqrt{\log(1/\err(h|k))} +
        \sqrt{\frac{\pi}{2}} \right)\\
      &\approx \sigma_k\left(\Phi^{-1}(\acc(h|k)) +
      \sqrt{\frac{\pi}{2}}\right).
      \end{split}
    \label{eq:radius_bound}
  \end{eqnarray}
 \label{thm:enflt}
\end{theorem}

\begin{proof}
  The main idea is to invoke Lemma \ref{thm:blowup}, and then apply the bound
  \eqref{eq:blowup} with $B = B(h,k) := \{x \in \mathcal X | h(x) \ne k\}$, $\mu
  = P_{X|k}$, and $c=\sigma_k^2$. See Appendix \ref{sec:proofs} for details.
\end{proof}

In the particular case of attacks happening in euclidean space (this is the
default setting in the literature), the above theorem has the following
corollary.
\begin{corollary}[Strong ``No Free Lunch'' Theorem on flat space]
  Let $ 1 \le q \le \infty$, and define
  \begin{eqnarray}
    \epsilon_q(h|k) := \epsilon(h|k)p^{\frac{1}{q}-\frac{1}{2}} \approx
    p^{\frac{1}{q}-\frac{1}{2}}\sigma_k\Phi^{-1}(\acc(h|k)).
  \end{eqnarray}
  If in addition to the assumptions of Theorem
  \ref{thm:enflt} the conditional data manifold $\mathcal X_k$ is flat, i.e
  $\operatorname{Ric}_{\mathcal X_k} = 0$, then for the $\ell_q$ threat model,
  we have
  
\textbf{\textit{(A1)} Adversarial robustness accuracy:} If $\epsilon \ge
\epsilon_q(h|k)$, then
  \begin{eqnarray}
    \acc_{\epsilon}(h|k) \le
    \min(\acc(h|k),e^{-\frac{p^{1-2/q}}{2\sigma_k^2}(\epsilon -
      \epsilon_q(h|k))^2}).
  \end{eqnarray}

\textbf{\textit{(A2)} Average distance to error set:} if $\epsilon \ge
\epsilon_q(h|k)$, then
  \begin{eqnarray}
    \begin{split}
      d(h|k) &\le \frac{\sigma_k}{p^{1/2-1/q}} \left(\sqrt{\log(1 /
          \err(h|k))} + \sqrt{\pi/2} \right)\\
      &=\frac{\sigma_k}{p^{1/2-1/q}} \left(\Phi^{-1}(\acc(h|k)) + \sqrt{\pi/2}
      \right) .
    \label{eq:radius_bound_q}
  \end{split}
  \end{eqnarray}

In particular, for the $\ell_\infty$ threat model,
we have

\textbf{\textit{(B1)} Adversarial robustness accuracy:} if $\epsilon \ge
\frac{\epsilon(h|k)}{\sqrt{p}}$, then
  \begin{eqnarray}
    \acc_{\epsilon}(h|k) \le \min(\acc(h|k),e^{-\frac{p}{2\sigma_k^2} (\epsilon
    - \epsilon(h|k)/\sqrt{p})^2}).
  \end{eqnarray}

\textbf{\textit{(B2)} Bound on average distance to error set:}
  \begin{eqnarray}
    \begin{split}
      d(h|k) &\le \frac{\sigma_k}{\sqrt{p}}\left(\sqrt{\log(1/\err(h|k))} +
        \sqrt{\pi/2} \right)\\
      &=\frac{\sigma_k}{\sqrt{p}} \left(\Phi^{-1}(\acc(h|k)) + \sqrt{\pi/2}
      \right) .
    \end{split}
    \label{eq:radius_bound_infty}
  \end{eqnarray}      
\label{thm:enflc}
\end{corollary}

\begin{proof}
  See Appendix \ref{sec:proofs}.
\end{proof}

\subsection{Making sense of the theorems}
Fig. \ref{fig:mnist} gives an instructive illustration of bounds in the above
theorems. For perfect classifiers, the test error $\err(h|k):=1-\acc(h|k)$ is
zero and so the factor $\sqrt{\log(1/\err(h|k))}$ appearing in definitions for
$\epsilon(h|k)$ and $\epsilon_q(h|k)$ is $\infty$; else this classifier-specific
factor grows only very slowly (the log function grows very slowly) as
$\acc(h|k)$ increases towards the perfect limit where $\acc(h|k)=1$. As
predicted by Corollary \ref{thm:enflc}, we observe in Fig. \ref{fig:mnist} that
beyond the critical value $\epsilon = \epsilon_\infty(h|k) :=
\sigma\sqrt{2\log(1/\err(h|k))/p}$, the
adversarial accuracy $\acc_\epsilon(h|k)$ decays at a Gaussian rate, and
eventually $\acc_\epsilon(h|k) \le \err(h|k)$ as soon as $\epsilon \ge
2\epsilon_\infty(h|k)$.


Comparing to the Gaussian special case (see section \ref{sec:apps} below), we
see that the curvature parameter $\sigma_k$ appearing in the theorems is an
analogue to the natural noise-level in the problem. The flat case $\mathcal X_k
= \mathbb R^p$ with an $\ell_\infty$ threat model is particularly instructive.
The critical values of $\epsilon$, namely $\epsilon_\infty(h|k)$ and
$2\epsilon_\infty(h|k)$ beyond which the compromising conclusions of the
Corollary \ref{thm:enflc} come into play is proportional to $\sigma_k/\sqrt{p}$.

Finally note that the $\ell_1$ threat model corresponding to $q=1$ in Corollary
\ref{thm:enflc}, is a convex proxy for the ``few-pixel'' threat model which was
investigated in ~\cite{onepixel}.

\subsection{Some consequences of the theorems}
\label{sec:apps}
 Recall that $P_{X|k}$ is the distribution of the inputs conditional on the
 class label being $k$ and $\mathcal X_k$ is the support of $P_{X|k}$. It turns
 out that the general ``No Free Lunch'' Theorem \ref{thm:enflt} and Corollary
 \ref{thm:enflc} apply to a broad range of problems, with certain geometric
 constraints on $P_{X|k}$ and $\mathcal X_k$. We discuss a non-exhaustive list
 of examples hereunder.

\subsubsection{Log-concave on a Riemannian manifold}
\label{sec:logcon}
Consider a conditional data model of the form $P_{X|k}\propto e^{-v_k(x)}dx$ 
on a $d$-dimensional Riemannian manifold $\mathcal X_k \subseteq \mathcal X$
satisfying the Bakry-Eme\'ry curvature condition~\cite{bakry85}
\begin{eqnarray}
  \operatorname{Hess}_x(v_k) +
  \operatorname{Ric}_x(\mathcal X) \succeq (1/\sigma^2_k)I_p,
  \label{eq:bakry}
\end{eqnarray}
for some $\sigma_k >
0$. Such a distribution is called \textit{log-concave}. By Corollary
1.1 of  ~\cite{otto2000} (and Corollary 3.2 of ~\cite{bobkov1999}), $P_{X|k}$ has the
$\text{T}_2(\sigma_k^2)$ property and therefore by Lemma \ref{thm:blowup}, the
BLOWUP($\sigma_k^2$) property, and Theorem \ref{thm:enflt} (and Corollary
\ref{thm:enflc} for flat space) applies.

The \textit{Holley-Stroock perturbation Theorem} ensures that if $P_{X|k} \propto
e^{-v_k(x) - u_k(x)}dx$ where $u_k$ is bounded, then Theorem \ref{thm:enflt}
(and Corollary \ref{thm:enflc} for flat space) holds with the noise parameter
$\sigma_k$ degraded to $\tilde{\sigma}_k := \sigma_k e^{\osc(u_k)}$, where
$\osc(u_k) := \sup_x u_k(x) - \inf_x u_k(x) \ge 0$.

\subsubsection{Elliptical Gaussian in euclidean space}
Consider the flat manifold $\mathcal X_k = \mathbb R^p$ and multi-variate
Gaussian distribution $P_{X|k} \propto e^{-v_k(x)}dx$ thereupon, where $v_k(x) =
\frac{1}{2}(x-m_k)^T\Sigma_k^{-1}(x-m_k)$, for some vector $m_k \in \mathbb R^p$
(called the mean) and positive-definite matrix $\Sigma_k$ (called the covariance
matrix) all of whose eigenvalues are $\le \sigma_k^2$. A direct computation
gives $\operatorname{Hess}(v_k) + \operatorname{Ric}_x \succeq 1/\sigma_k^2 + 0
= 1/\sigma_k^2$ for all $x \in \mathbb R^p$. So this is an instance of the above
log-concave example, and so the same bounds hold. Thus we get an elliptical
version (and therefore a strict generalization) of the basic ``No Free Lunch''
theorem in ~\cite{tsipras18}, with exactly the same constants in the bounds.
These results are also confirmed empirically in section \ref{sec:sim}.

\subsubsection{Volume measure on compact Riemannian manifold with positive Ricci
  \label{sec:uniform}
  curvature} Indeed the distribution $dP_{X|k} = dvol(\mathcal
X_k)/vol({\mathcal X_k})$ is log-concave on $\mathcal X_k$ since it satisfies
the Bakry-Eme\'ry curvature condition \eqref{eq:bakry} with $v_k = 0$ and
$\sigma_k = 1/\sqrt{R_k}$ where $R_k>0$ is the minimum of the Ricci curvature on
$\mathcal X_k$. Thus by section \ref{sec:logcon}, it follows that our
theorems hold. A prime example of such a manifold is a $p$-sphere of radius $r >
0$, thus with constant Ricci curvature $(p-1)/r^2$. For this example
~\cite{gilmerspheres18} is an instance (more on this in section
\ref{sec:related-work}).

\paragraph*{Application to the ``Adversarial Spheres'' problem.}
In the recent recent ``Adversarial Spheres'' paper ~\cite{gilmerspheres18}, wherein
the authors consider a 2-class problem on a so-called ``concentric spheres''
dataset. This problem  can be described in our notation as: $P_{X|+} =$ uniform
distribution on $p$-dimensional sphere of radius $r_+$ and $P_{X|-} =$ uniform
distribution on $p$-dimensional sphere of radius $r_-$. The classification
problem is to decide which of the two concentric spheres a sampled point came
from. The authors (see Theorem  5.1 of mentioned paper) show that for this
problem, for large $p$ and for any non-perfect classifier $h$, the average
$\ell_2$ distance between a point and the set of misclassified points is bounded
as follows
\begin{eqnarray}
  \ell_2(h) = \mathcal O(\Phi^{-1}(\acc(h|k))/\sqrt{p}).
  \label{eq:gilmer}
\end{eqnarray}

We now show how to obtain the above bound via a direct application of our
theorem \ref{thm:enflt}. Indeed, since these spaces are clearly compact
Riemannian manifolds with constant curvature $(p-1)/r^2_k$, each $P_{X|k}$ is and
instances of \ref{sec:uniform}, and so satisfies
$\text{T}_2\left(\frac{r_k^2}{p-1}\right)$.
Consequently, our Theorem \ref{thm:enflt} kicks-in and bound the average
distance of sample points with true label $k \in \{\pm\}$, to the error set (set
of misclassified samples):
\begin{itemize}
  \item $d_{\text{geo}}(h|k) \le
\frac{r_k}{\sqrt{p-1}}(\sqrt{2\log(1/\err(h|k))}) + \sqrt{\pi/2})$
for the geodesic threat model (induced by the ``great circle'' distance
between points), and
\item $\ell_2(h|k) \le \frac{r_k}{\sqrt{p-1}}(\sqrt{2\log(1/\err(h|k))} +
  \sqrt{\pi/2})$
  for the  $\ell_2$ threat  model. This follows from the previous inequality
  because the geodesic (aka great circle) distance between two points on a
  sphere is always larger than the euclidean $\ell_2$ distance between points.
 \end{itemize}

 To link more explicitly  with the bound \eqref{eq:gilmer} proposed in
 ~\cite{gilmerspheres18}, one notes the following elementary (and very crude)
 approximation of Gaussian quantile function
 $\Phi^{-1}(a) \approx \sqrt{2\log (1/(1-a))}$ for $a \in [0, 1)$. Thus,
 $\Phi^{-1}(1-\err(h|k))/\sqrt{p}$  and $\sqrt{2\log(1/\err(h|k))/(p-1)}$ are of the
 same order, for large $p$. Consequently, our bounds can be seen as a strict
 generalization of the bounds in ~\cite{gilmerspheres18}.

\subsubsection{Lipschitz pushforward of a $\text{T}_2$ distribution}
Lemma 2.1 of ~\cite{djellout2004} ensures that if $P_{X|k}$ is the
\textit{pushforward} via an $L_k$-Lipschitz map  $T_k: \mathcal Z_k \rightarrow
\mathcal X_k$ between metric spaces (an assumption which is implicitly made when
machine learning practitioners model images using generative neural
networks\footnote{The Lipschitz constant of a feed-forward neural network with
  1-Lipschitz activation function, e.g ReLU, sigmoid, etc., is bounded by the
  product of operator norms of the layer-to-layer parameter matrices.}, for
example), of a distribution $\mu_k$ which satisfies
$\text{T}_2(\tilde{\sigma}_k^2)$ on $\mathcal Z_k$ for some $\tilde{\sigma}_k >
0$, then $P_{X|k}$ satisfies $\text{T}_2(L_k^2\tilde{\sigma}_k^2)$ on $\mathcal
X_k$, and so Theorem \ref{thm:enflt} (and Corollary \ref{thm:enflc} for flat
space) holds with $\sigma_k=L_k\tilde{\sigma}_k$. This is precisely the GAN-type
data model assumed by ~\cite{fawzi18}, with $\mathcal Z_k:=\mathbb R^{p'}$  and
$\mu_k = \mathcal N(0,\sigma I_{p'})$ for all $k$.

\section{Related works}
\label{sec:related-work}
There is now a rich literature trying to understand adversarial robustness. Just
to name a few, let us mention ~\cite{tsipras18, schmidt2018, bubeck2018,
  gilmerspheres18, fawzi18, saeed2018, sinha2018}.
Below, we discuss a representative subset of these works, which is most relevant
to our own contributions presented in this manuscript. These all use some kind
of Gaussian isoperimetric inequality~\cite{boucheron2013}, and turn out to be
\textbf{very special cases} of the general bounds presented in Theorem \ref{thm:enflt}
and Corollary \ref{thm:enflc}. See section \ref{sec:apps} for a detailed
discussion on generality our results.

\paragraph*{Gaussian and Bernoulli models.}
We have already mentioned the work ~\cite{tsipras18}, which first showed
that motivating problem presented in section \ref{sec:toy}, every classifier can
be fooled with high probability. In a followup paper~\cite{schmidt2018}, the
authors have also suggested that the sample complexity for robust generalization
is much higher than for standard generalization. These observations are also
strengthened by independent works of ~\cite{bubeck2018}.

\paragraph*{Adversarial spheres.}
The work which is most similar in flavor to ours is the recent ``Adversarial
Spheres'' paper ~\cite{gilmerspheres18}, wherein the authors consider a 2-class
problem on classifying two concentric spheres in $\mathbb R^p$ of different
radii. The authors showed that the distance of each point to the set of
misclassified images is of order $\mathcal O(1/\sqrt{p})$. We discussed this
work in detail in section \ref{sec:apps} and showed that it follows directly
from our Theorem \ref{thm:enflt}.

\paragraph*{Hyper-cubes.}
Results for the the hyper-cube image model considered in ~\cite{goldstein} can
be recovered by noting that the uniform measure on $[0, 1]^p$ is the pushforward
$T_\#\gamma_p$ of the standard Gaussian on $\mathbb R^p$, where $T:\mathbb R^p
\rightarrow [0, 1]^p$ is $(2\pi)^{-1/2}$-Lipschitz map defined by
$T(z_1,\ldots,z_p):=(\Phi(z_1),\ldots,\Phi(z_p))$.

\paragraph*{Generative models.}
In ~\cite{fawzi18}, the authors considered a scenario where data-generating
process is via passing a multivariate Gaussian distribution through a Lipschitz
continuous mapping
$g:
\mathbb R^{m} \rightarrow \mathcal X$, called the generator. The authors then
studied the per-sample robustness radius defined by $r_{\mathcal X}(x,k) := \inf
\{\|x'-x\|_2 \text{ s.t } x' \in \mathcal X,\; h(x') \ne k\}$. In the notation
of our manuscript, this can be rewritten as $r_{\mathcal X}(x,k) := d_{\mathcal
  X}(x,B(h, k))$, from which it is clear that $r_{\mathcal X}(x,k) \le \epsilon$
iff $x \in B(h,k)^\epsilon$. Using the basic Gaussian isoperimetric
inequality~\cite{boucheron2013}, the authors then proceed to obtain bounds on
the probability that the classifier changes its output on an
$\epsilon$-perturbation of some point on manifold the data manifold, namely
$\acc^{\text{switch}}_\epsilon(h):=1-\sum_{k}\pi_k\err^{\text{switch}}_\epsilon(h|k)
$, where
$\err^{\text{switch}}_\epsilon(h|k):=P_{X|k}(C_{k\rightarrow}(\epsilon)) =
\acc(h|k)\err_\epsilon(h|k)$ and $C_{k\rightarrow}(\epsilon) := B(h,k)^\epsilon
- B(h,k)$ is the annulus in Fig.  \ref{fig:blowup}. Our bounds in Theorem
\ref{thm:enflt} and Corollary \ref{thm:enflc} can then be seen as generalizing
the methods and bounds in ~\cite{fawzi18} to more general data distributions
satisfying $W_2$ transportation-cost inequalities $\text{T}_2(c)$, with $c > 0$.

\paragraph*{Distributional robustness and regularization.} On a completely
different footing, ~\cite{blanchet2016,esfahani2017,sinha2018}
have linked distributional robustness to robust estimation theory from
classical statistics and regularization. An interesting bi-product of these
developments is that penalized regression problems like the square-root Lasso
and sparse logistic regression have been recovered as distributional robust
counterparts of the unregularized problems.

\section{Experimental evaluation}
\label{sec:experiments}
We now present some empirical validation for our theoretical results.
\subsection{Simulated data}
\label{sec:sim}
The simulated data are discussed in section \ref{sec:toy}: $Y \sim \Bern(\{\pm
1\})$, $X|Y \sim \mathcal N (Y\eta, 1)^{\times p}$, with $p=1000$ where $\eta$
is an SNR parameter which controls the difficulty of the problem. Here, The classifier $h$ is a
multi-layer perceptron with architecture $1000 \rightarrow 200 \rightarrow 100
\rightarrow 2$ and ReLU activations.  The results
are are shown in Fig. \ref{fig:mnist} (Left). As predicted by our theorems, we observe
that beyond the critical value $\epsilon = \epsilon_\infty(h) :=
\sigma\sqrt{2\log(1/\err(h))/p} = \tilde{\mathcal O}(\sigma/\sqrt{p})$, where
$\err(h):= 1 - \acc(h)$, the adversarial accuracy $\acc_\epsilon(h)$ decays
exponential fast, and passes below the horizontal line $\err(h)$ as soon as
$\epsilon \ge 2 \epsilon_\infty(h)$.
\begin{figure}[!htpb]
  \includegraphics[width=.235\textwidth]{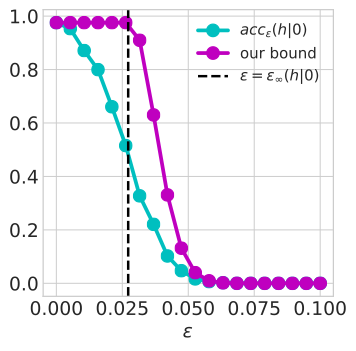}
  \includegraphics[width=.235\textwidth]{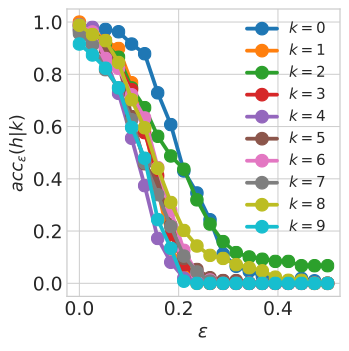}
  \caption{Illustrating The Extended ``No Free Lunch'' Theorem
    \ref{thm:enflc} for the $\ell_\infty$ threat model on the
    classification problems. \textbf{Left:} Simulated data ~\cite{tsipras18} (discused in section
    \ref{sec:toy}) with with $p=10000$ and SNR parameter $\eta=1$.
    The classifier $h$ is a multi-layer perceptron with architecture $10000 \rightarrow 200
    \rightarrow 100 \rightarrow 2$.
    As $\epsilon$ is increased, the robust accuracy (only shown here for the
    class $k=0$, but results for the class $k=1$ are similar)
    degrades slowly and then eventually hits a phase-transition
    point $\epsilon=\epsilon_\infty(h|0)$; it then decays exponentially fast,
    and the performance is eventually reduced to chance level. This is as
    predicted by our Theorems \ref{thm:enflt} and \ref{thm:enflc}.
    \textbf{Right:} MNIST dataset. The classifier $h$ is a deep feed-forward CNN ($Conv2d \rightarrow
    Conv2d \rightarrow 320 \rightarrow 10$) is trained
    using PyTorch \url{https://pytorch.org/} to predict MNIST classification
    problem. The pattern of decay of the adversarial accuracy similar to that on
    the simulated, indicating that this real dataset might also suffer from
    concentration, allowing our theorems to apply.
  }
  \label{fig:mnist}
\end{figure}

\subsection{Real data}
Wondering whether the phase transition and bounds predicted by Theorem
\ref{thm:enflt} and Corollary \ref{thm:enflt} holds for real data, we trained a
deep feed-forward CNN (architecture: $Conv2d \rightarrow
    Conv2d \rightarrow 320 \rightarrow 10$) for classification on the MNIST
    dataset~\cite{mnist}, a standard benchmark problem in supervised
    machine-learning. The results are shown in Fig. \ref{fig:mnist}.
This model attains a classification accuracy of 98\% on held-out data. We
consider the performance of the model on adversarialy modified images according
to the $\ell_\infty$ threat model, at a given tolerance level (maximum allowed
modification per pixel) $\epsilon$. As $\epsilon$ is increased, the performance
degrades slowly and then eventually hits a phase-transition point; it then decays
exponentially fast and the performance is eventually reduced to chance level.
This behavior is in accordance with Corollary \ref{thm:enflc}, and suggests that
the range of applicability of our results may be much larger than what we have
been able to theoretically establish in Theorem \ref{thm:enflt} and Corollary
\ref{thm:enflc}.

Of course, a more extensive experimental study would be required to strengthen
this empirical observation.

\section{Concluding remarks}
\label{sec:conclusion}
We have shown that on a very broad class of data distributions, any classifier
with even a bit of accuracy is vulnerable to adversarial attacks. Our work uses
powerful tools from geometric probability theory to generalize all the main
impossibility results that have appeared in adversarial robustness literature.
Moreover, our results would encourage one to conjecture that the modulus of
concentration of probability distribution (e.g in $\text{T}_2$ inequalities) on
a manifold completely characterizes the adversarial or distributional robust
accuracy in classification problems.

\paragraph*{Redefine the rules of the game ?}
A limitation for adversarial robustness, as universal our strong No Free
Lunch Theorem we have developed in this paper could indicate that the attack
models currently being considered in the literature, namely additive
perturbations measured in the $\ell_0$, $\ell_1$, $\ell_2$, $\ell_\infty$, etc.
norms, and in which the attacker can make as many queries as they which, may be
too lax. Just like in coding theory where a rethinking of the constraints on a
channel leads to the Shannon limit to be improved, one could hope that a careful
rethink of the constraints put on the adversarial attacker might alleviate the
pessimistic effect of our impossibility results. As remarked in
\cite{gilmergame18}, it is not even clear if the current existing attack models
are most plausible.

\paragraph*{Future directions.}
One could consider the following open questions, as natural continuation of our
work:
\begin{itemize}
\item Extend Theorem \ref{thm:enflt} and Corollary \ref{thm:enflc} to more
  general data distributions.
 \item Study more complex threat models, e.g small deformations.  
\item Fine grained analysis of sample complexity and complexity of hypotheses
  class, with respect to adversarial and distributional robustness. This
  question has been partially studied in ~\cite{schmidt2018,bubeck2018} in the
  adversarial case, and  ~\cite{sinha2018} in the distributional robust scenario.
  \item Study more general threat models. \cite{gilmergame18} has argued that
    most of the proof-of-concept problems studied in theory papers might not be
    completely aligned with real security concerns faced by machine learning
    applications. It would be interesting to see how the theoretical bounds
    presented in our manuscript translate on real-world datasets, beyond the
    MNIST on which we showed some preliminary experimental results.
\item Develop more geometric insights linking adversarial robustness and curvature
  of decision boundaries. This view was first introduced in ~\cite{fawzi18b}.
\end{itemize}

\paragraph{Acknowledgments.} I would wish to thank Noureddine El Karoui for
  stimulating discussions;  Alberto Bietti and Albert Thomas for their useful
  comments and remarks.


\nocite{langley00}

\bibliography{bib.bib}
\bibliographystyle{icml2019}

\cleardoublepage
\twocolumn[
\icmltitle{(Supplementary material) Limitations of Adversarial Robustness: Strong No Free Lunch Theorem}
\icmlsetsymbol{equal}{*}

\begin{icmlauthorlist}
  \icmlauthor{Elvis Dohmatob}{aff}
\end{icmlauthorlist}

\icmlaffiliation{aff}{Criteo, Paris, France.}


\icmlkeywords{boring formatting information, machine learning, ICML}

\vskip 0.3in
]

\printAffiliationsAndNotice 

\appendix


\section{Proofs}
\label{sec:proofs}

\begin{proof}[Proof of claims on the toy problem from section \ref{sec:toy}]
  These claims were already proved in ~\cite{tsipras18}. We provide a proof here
  just for completeness.

  Now, one computes
\begin{eqnarray*}
\begin{split}
&\acc(h_{\text{avg}}) := \mathbb P_{(X,Y)}\left(h_{\text{avg}}(X)
  = Y\right) = \mathbb P\left(Y{w}^TX \ge 0\right)\\
&= \mathbb P_Y\left((Y/(p-1))\sum_{j\ge2}\mathcal N(\eta Y,1) \ge 0\right)\\
&=\mathbb P\left(\mathcal N(\eta,1/(p-1)) \ge 0\right) =
\mathbb P\left(\mathcal N(0,1/(p-1)) \ge -\eta\right)\\
&=\mathbb P\left(\mathcal N(0,1/(p-1)) \le \eta \right) \ge 1 -
e^{-{(p-1)\eta^2}/{2}},
& 
\end{split}
\end{eqnarray*}
which is $\ge 1 - \delta$ if $\eta \ge \sqrt{2\log(1/\delta)/(p-1)}$.
Likewise, for $\epsilon \ge \eta$, it was shown in ~\cite{tsipras18} that the
adversarial robustness accuracy of $h_{\text{avg}}$ writes
\begin{eqnarray*}
\begin{split}
&\acc_\epsilon(h_{\text{avg}}) := \mathbb
P_{(X,Y)}\left(Yh_{\text{avg}}(X + {\Delta x}) \ge 0\; \forall \|{\Delta
    x}\|_\infty \le \epsilon\right)\\
&=\mathbb P_{(X,Y)}\left( \inf_{\|{\Delta x}\|_\infty \le \epsilon}Yw^T({X
    +\Delta x}) \ge 0\right)\\
&= \mathbb P_{(X,Y)}\left( Y{w}^TX - \epsilon \|Y{w}\|_1 \ge 0\right)\\
&= \mathbb P_{(X,Y)}\left(Y{w}^TX - \epsilon \ge 0\right)\\
&= \mathbb P(\mathcal N(0,1/(p-1))
\ge \epsilon- \eta)
\le e^{-{(p-1)(\epsilon- \eta)^2}/{2}}.
\end{split}
\end{eqnarray*}
Thus $\acc_\epsilon(h_{\text{avg}}) \le \delta$ for $\epsilon \ge \eta +
\sqrt{2\log(1/\delta)/(p-1)}$,
which completes the proof.
\end{proof}

\begin{proof}[Proof of Theorem \ref{thm:enflt}]
  Let $h: \mathcal X \rightarrow \{1,\ldots,K\}$ be a classifier, and for a
  fixed class label $k \in \{1,2,\ldots,K\}$, define the set $B(h,k) := \{x \in
  \mathcal X|h(x) \ne k\}$. Because we only consider $P_{X|Y}$-a.e continuous
  classifiers, each $B(h,k)$ is Borel. Conditioned on the event ``$y=k$'', the
  probability of $B(h,k)$ is precisely the average error made by the classifier
  $h$ on the class label $k$. That is, $\acc(h|k) = 1-P_{X|k}(B(h,k))$.
  Now, the assumptions imply by virtue of Lemma \ref{thm:blowup}, that $P_{X|k}$
  has the BLOWUP($c$) property. Thus, if $\epsilon \ge
  \sigma_k\sqrt{2\log(1/(P_{X|Y}(B(h,k))} = \sigma_k\sqrt{2\log(1/\err(h|k)} =:
  \epsilon(h|k)$, then one has
  \begin{eqnarray*}
    \begin{split}
      &\acc_\epsilon(h|k) =
      1-P_{X|k}(B(h,k)_{d_{\text{geo}}}^\epsilon)\\
      &\le e^{-\frac{1}{2\sigma_k^2}(\epsilon-\sigma_k\sqrt{2\log(1/(P_{X|k}(B(h,k))})^2}\\
      &= e^{-\frac{1}{2\sigma_k^2}(\epsilon-\sigma_k\sqrt{2\log(1/\err(h|k)})^2}
      = e^{-\frac{1}{2\sigma_k^2}(\epsilon-\epsilon(h|k))^2}\\
      &\le e^{-\frac{1}{2\sigma_k^2}\epsilon(h|k)^2} =\err(h|k),\text{ if
      }\epsilon \ge 2\epsilon(h|k).
    \end{split}
  \end{eqnarray*}
  On the other hand, it is clear that $\acc_\epsilon(h|k) \le \acc(h|k)$ for any
  $\epsilon \ge 0$ since $B(h,k) \subseteq B(h,k)^\epsilon$ for any threat
  model. This concludes the proof of part \textit{(A)}. For part \textit{(B)},
  define the random variable $Z:=d(X,B(h,k))$ and note that
  \begin{eqnarray*}
    \begin{split}
      &d(h|k):=\mathbb E_{X|k}[d(X, B(h,k))] = \int_0^\infty P_{X|k}(Z \ge
      \epsilon)d\epsilon\\
      &=\int_0^{\epsilon(h|k)} P_{X|k}(Z \ge \epsilon) d\epsilon +
      \int_{\epsilon(h|k)}^\infty P_{X|k}(Z \ge \epsilon)d\epsilon\\
      &\le \epsilon(h|k) + \int_{\epsilon(h|k)}^\infty P_{X|k}(Z \ge
      \epsilon)d\epsilon,\;\text{ as }P_{X|k}(Z \ge \epsilon) \le 1\\
      &\le \epsilon(h|k) + \int_{\epsilon(h|k)}^\infty
      e^{-\frac{1}{2\sigma_k^2}(\epsilon -
        \epsilon(h|k))^2}d\epsilon,\;\text{ by inequality \eqref{eq:bingo}}\\
      &= \epsilon(h|k) + \frac{\sigma_k\sqrt{2\pi}}{2}\left(\int_{-\infty}^\infty
      \frac{1}{\sigma_k\sqrt{2\pi}}e^{-\frac{1}{2\sigma_k^2}\epsilon^2}d\epsilon\right)
\\
&=\epsilon(h|k) + \frac{\sigma_k\sqrt{2\pi}}{2} = \sigma_k
\left(\sqrt{\log(1/\err(h|k))}+\sqrt{\frac{\pi}{2}}\right),
    \end{split}
  \end{eqnarray*}
  which is the desired inequality.
\end{proof}
 
\begin{proof}[Proof of Corollary \ref{thm:enflc}]
  For flat geometry $\mathcal X_k = \mathbb R^p$; part \textit{(A1)} of
  Corollary \ref{thm:enflc} then follows from Theorem \ref{thm:enflt} and the
  equivalence of $\ell_q$ norms, in particular
  \begin{eqnarray}
    \|x\|_2 \le  p^{1/2-1/q}\|x\|_q,
    \label{eq:equivalence_lq}
  \end{eqnarray}
  for all $x \in \mathbb R^p$ and for all $q \in [1, \infty]$. Thus we have the
  blowup inclusion $B(h,k)_{\ell_2}^{\epsilon p^{1/2-1/q}} \subseteq
  B(h,k)_{\ell_q}^\epsilon$. Part \textit{(B1)} is just the result restated for
  $q=\infty$. The proofs of parts \textit{(A2)} and \textit{(B2)} trivially
  follow from the inequality \eqref{eq:equivalence_lq}.
\end{proof}

\begin{remark}
  Note that the particular structure of the error set $B(h,k)$ did not play any
  part in the proof of Theorem \ref{thm:enflt} or of Corollary \ref{thm:enflc},
  beyond the requirement that the set be Borel. This means that we can obtain
  and prove analogous bounds for much broader class of losses. For example, it
  is trivial to extend the theorem to targeted attacks, wherein the attacker can
  aim to change an images label from $k$ to a particular $k'$.
\end{remark}

\begin{proof}[Proof of Lemma \ref{thm:blowup}]
  Let $B$ be a Borel subset of $\mathcal X=(\mathcal X,d)$ with $\mu(B) > 0$,
  and let $\mu|_B$ be the restriction of $\mu$ onto $B$ defined by $\mu|_B(A) :=
  \mu(A \cap B)/\mu(B)$ for every Borel $A \subseteq \mathcal X$. Note that
  $\mu|_B \ll \mu$ with \textit{Radon-Nikodym derivative} $\frac{d \mu|_B}{d\mu}
  = \frac{1}{\mu(B)}1_B$. A direct computation then reveals that
  \begin{eqnarray*}
    \begin{split}
      \kl(\mu|_B\|\mu) &= \int \log\left(\frac{d \mu|_B}{d \mu}\right)d\mu|_B\\
      &= \int_B \log\left(\frac{1}{\mu(B)}\right)d\mu|_B\\
      &= \log(1/\mu(B))\mu|_B(B) = \log\left(\frac{1}{\mu(B)}\right).
    \end{split}
  \end{eqnarray*}
  On the other hand, if $X$ is a random variable with law $\mu|_B$ and $X'$ is a
  random variable with law $\mu|_{\mathcal X\setminus B^\epsilon}$, then the
  definition of $B^\epsilon$ ensures that $d(X,X') \ge \epsilon$ $\mu$-a.s, and
  so by definition \eqref{eq:wass}, one has $W_2(\mu|_B,\mu|_{\mathcal X\setminus
    B^\epsilon}) \ge \epsilon$. Putting things together yields
  \begin{eqnarray*}
    \begin{split}
      \epsilon &\le W_2(\mu|_B,\mu_{\mathcal X\setminus B^\epsilon})
      \le W_2(\mu|_B,\mu) +
      W_2(\mu|_{\mathcal X\setminus
        B^\epsilon},\mu)\\
      & \le \sqrt{2c\kl(\mu|_B\|\mu)} + \sqrt{2c\kl(\mu|_{\mathcal X\setminus
          B^\epsilon}\|\mu)}\\
      & \le \sqrt{2c\log(1/\mu(B))} + \sqrt{2c\log(1/\mu(\mathcal X\setminus
        B^\epsilon))}\\
      & = \sqrt{2c\log(1/\mu(B))} + \sqrt{2c\log(1/(1-\mu(B^\epsilon))},
    \end{split}
  \end{eqnarray*}
  where the first inequality is the triangle inequality for $W_2$ and the second
  is the $\text{T}_2(c)$ property assumed in the Lemma. Rearranging the above
  inequality gives
$$
\sqrt{2c\log(1/(1-\mu(B^\epsilon)))} \ge \epsilon - \sqrt{2c\log(1/\mu(B))},
$$
Thus, if $\epsilon \ge \sqrt{2c\log(1/\mu(B))}$, we can square both sides,
multiply by $c/2$ and apply the increasing function $t \mapsto e^t$, to get the
claimed inequality.
\end{proof}


\section{Distributional No ``Free Lunch'' Theorem}
\label{sec:dr}
As before, let $h:\mathcal X \rightarrow \mathcal Y$ be a classifier 
and $\epsilon \ge 0$ be a tolerance level. Let
$\widetilde{\acc}_\epsilon(h)$ denote the \textit{distributional robustness accuracy} of
$h$ at tolerance $\epsilon$, that is the worst possible
classification accuracy at test time, when the conditional distribution $P$ is
changed by at most $\epsilon$ in the Wasserstein-1 sense. More precisely,
  \begin{eqnarray}
    \widetilde{\acc}_\epsilon(h) := \inf_{Q \in \mathcal P(\mathcal X \times
    \mathcal Y),\;W_1(Q, P) \le \epsilon} Q(h(x) = y),
  \end{eqnarray}
  where the Wasserstein $1$-distance $W_1(Q,P)$ (see equation \eqref{eq:wass}
  for definition) in the constraint is with respect to the pseudo-metric
  $\tilde{d}$ on $\mathcal X \times \mathcal Y$ defined by
  $$
  \tilde{d}((x',y'), (x,y)) := \begin{cases}d(x',x),&\mbox{ if }y'=y,\\
    \infty,&\mbox{ else.}\end{cases}
  $$
The choice of $\tilde{d}$ ensures that we only consider alternative
distributions that conserve the marginals $\pi_y$; robustness is only
considered w.r.t to changes in the class-conditional distributions $P_{X|k}$.

Note that we can rewrite $\widetilde{\acc}_\epsilon(h) = 1 -
\widetilde{\err}_\epsilon(h)$,
\begin{eqnarray}
\widetilde{\err}_\epsilon(h) :=  \sup_{Q \in \mathcal P(\mathcal X \times
  \mathcal Y),\;W_1(Q, P) \le \epsilon} Q(X \in B(h,Y)),
\end{eqnarray}
where is the \text{distributional robustness test error} and  $B(h,y) := \{x \in
\mathcal X | h(x) \ne y\}$ as before. Of course, the goal of a machine learning
algorithm is to select a classifier (perhaps from a restricted family)
for which the average adversarial accuracy $\acc_{\epsilon}(h)$ is maximized.
This can be seen as a two player game: the machine learner chooses a strategy
$h$, to which an adversary replies by choosing a perturbed version  $Q \in
\mathcal P(\mathcal X \times \mathcal Y)$ of the data distribution, used to
measure the bad event ``$h(X) \ne Y$''.


It turns out that the lower bounds on adversarial accuracy obtained in Theorem
\ref{thm:enflt} apply to distributional robustness as well.
\begin{corollary}[No ``Free Lunch'' for distributional robustness]
  Theorem \ref{thm:enflt} holds for distributional robustness, i.e with
  $\acc_\epsilon(h|k)$ replaced with $\widetilde{\acc}_\epsilon(h|k)$.
  \label{thm:dr}
\end{corollary}

\begin{proof}
  See Appendix \ref{sec:proofs}.
\end{proof}

\begin{proof}[Proof of Corollary \ref{thm:dr}]
  We will use a dual representation of $\widetilde{\acc}_\epsilon(h|k)$ to
  establish that $\widetilde{\acc}_\epsilon(h|k) \le \acc_\epsilon(h|k)$. That
  is, distributional robustness is harder than adversarial robustness. In
  particular, this will allow us apply the lower bounds on adversarial accuracy
  obtained in Theorem \ref{thm:enflt} to distributional robustness as well!

  So, for $\lambda \ge 0$, consider the convex-conjugate of $(x,y) \mapsto 1_{x
    \in B(h,y)}$ with respect to the pseudo-metric $\tilde{d}$, namely
  \vspace{-.4cm}
$$
1_{x \in B(h,y)}^{\lambda \tilde{d}} := \sup_{(x',y') \in \mathcal X \times
  \mathcal Y}1_{x' \in B(h)}-\lambda \tilde{d}((x',y'), (x,y)).
  $$
  A straightforward computation gives
  \[
    \begin{split}
      &1_{x \in B(h,y)}^{\lambda \tilde{d}} := \sup_{(x',y') \in \mathcal X
        \times
        \mathcal Y}1_{x' \in B(h,y')}-\lambda \tilde{d}((x',y'), (x,y))\\
      &= \max_{B \in \{B(h,y),\;\mathcal X\setminus B(h,y)\}}\sup_{x' \in
        B}1_{x' \in B(h,y)}-\lambda d(x', x)\\
      &= \max(1-\lambda d(x,B(h,y)), -\lambda d(x,\mathcal X\setminus
      B(h,y)))\\
      &=(1 - \lambda d(x, B(h,y)))_+.
    \end{split}
  \]
  Now, since the transport cost function $\tilde{d}$ is nonnegative and
  lower-semicontinuous, 
  strong-duality holds~\cite{Villani,blanchet2016} and one has
  \begin{eqnarray*}
    \begin{split}
      &\sup_{W_1(Q, P) \le \epsilon} Q(h(X) \ne Y)
      \\
      &= \inf_{\lambda \ge 0}\sup_Q(Q(X \in B(h,Y))+\lambda
      (\epsilon-W_1(Q,P)))\\
      &=\inf_{\lambda \ge 0}\left(\sup_Q(Q(X \in B(h,Y))-\lambda
        W_1(Q,P)) + \lambda\epsilon\right)\\
      &=\inf_{\lambda \ge 0}(\mathbb E_{(x,y) \sim P}[1_{x \in B(h,y)}^{\lambda
        \tilde{d}}] + \lambda\epsilon)\\ &= \inf_{\lambda \ge 0}(\mathbb
      E_{(x,y) \sim P}[(1-\lambda d(x,
      B(h,y)))_+] + \lambda\epsilon)\\
      &= P(X \in B(h,Y)^{\lambda_*^{-1}}),
    \end{split}
  \end{eqnarray*}
  where $\lambda_* = \lambda_*(h) \ge 0$ is the (unique!) value of $\lambda$ at
  which the infimum is attained and we have used the previous computations and
  the handy formula
  \begin{eqnarray*}
    \begin{split}
      \sup_Q(Q(X \in B(h,Y))-\lambda W_1(Q,P)) &= \mathbb E_{P}[1_{X \in
        B(h,Y)}^{\lambda \tilde{d}}],
    \end{split}
  \end{eqnarray*}
  which is a direct consequence of Remark 1 of  ~\cite{blanchet2016}. Furthermore,
  by Lemma 2 of ~\cite{blanchet2016}, one has
  \begin{eqnarray*}
    \begin{split}
      \epsilon &\le \sum_k
      \pi_k\int_{B(h,k)^{\lambda_*^{-1}}}d(x,B(h,k))dP_{X|k}(x)\\
      &\le \sum_k \pi_k \lambda_*^{-1}P_{X|k}(X \in B(h,k)^{\lambda_*^{-1}})\\
      &= \lambda_*^{-1}P(X \in B(h,Y)^{\lambda_*^{-1}}) \le \lambda_*^{-1}.
    \end{split}
  \end{eqnarray*}
  Thus $\lambda_*^{-1} \ge \epsilon$ and combining with the previous
  inequalities gives

  \begin{eqnarray*}
    \begin{split}
      \sup_{Q \in \mathcal P(\mathcal X),\;W_1(Q, P) \le \epsilon} Q(h(X) \ne Y)
      &\ge P(X \in B(h,Y)^{\lambda_*^{-1}})\\
      &\ge P(X \in B(h,Y)^\epsilon).
    \end{split}
  \end{eqnarray*}
  Finally, noting that $\acc_\epsilon(h) = 1 - P(X \in B(h,Y)^\epsilon)$, one
  gets the claimed inequality $\widetilde{\acc}_\epsilon(h) \le
  \acc_\epsilon(h)$.
\end{proof}

\end{document}